\documentclass[11pt,letterpaper]{article}

\usepackage[utf8]{inputenc}
\usepackage[T1]{fontenc}
\usepackage[english]{babel}
\usepackage[in]{fullpage}
\usepackage{matharticle}
\usepackage[noend]{algpseudocode}
\usepackage[textsize=scriptsize]{todonotes}
\usepackage[hidelinks]{hyperref}
\usepackage{booktabs}
\usepackage{thm-restate}
\usepackage{subcaption}
\usepackage{etoolbox}
\usepackage{tikz}
\usepackage{pgfplots}
\pgfplotsset{compat=1.9}

\usepackage{algorithm}

\newcommand{\Otilde}{\widetilde{O}}
\newcommand{\flatten}[1]{\overline{#1}}
\newcommand{\vol}{\mathrm{vol}}
\newcommand{\supp}{\mathrm{supp}}
\newcommand{\eps}{\varepsilon}
\newcommand{\rect}{\mathcal{R}}
\newcommand{\hier}{\mathcal{D}}
\newcommand{\calH}{\mathcal{H}}
\newcommand{\calT}{\mathcal{T}}
\newcommand{\calF}{\mathcal{F}}
\newcommand{\calJ}{\mathcal{J}}
\newcommand{\calX}{\mathcal{X}}
\newcommand{\calG}{\mathcal{G}}
\newcommand{\calS}{\mathcal{S}}
\newcommand{\setI}{\mathcal{I}}
\newcommand{\setJ}{\mathcal{J}}
\newcommand{\hhat}{\widehat{h}}
\newcommand{\fhat}{\widehat{f}}
\newcommand{\VC}{\mathrm{VC}}
\newcommand{\A}{\mathcal{A}}
\newcommand{\tOPT}{\widetilde{\OPT}}
\newcommand{\grid}{\mathrm{grid}}

\def\colorful{0}

\def\nnewcolor{1}
\ifnum\nnewcolor=1

\fi
\ifnum\nnewcolor=0

\fi

\ifnum\colorful=1
\newcommand{\new}[1]{{\color{red} #1}}
\newcommand{\blue}[1]{{\color{blue} #1}}

\else
\newcommand{\new}[1]{{#1}}
\newcommand{\blue}[1]{{#1}}

\fi

\title{Fast and Sample Near-Optimal Algorithms\\ 
for Learning Multidimensional Histograms}

\author{
Ilias Diakonikolas\thanks{Supported by NSF Award CCF-1652862 (CAREER) and a Sloan Research Fellowship.}\\
USC\\
\texttt{diakonik@usc.edu}
\and
Jerry Li \thanks{Supported by NSF grant CCF-1217921, DOE grant de-sc0008923, NSF CAREER Award CCF-145326, and a NSF Graduate Research Fellowship}\\
MIT\\
\texttt{jerryzli@mit.edu}\\
\and
Ludwig Schmidt\thanks{Supported by a Google PhD Fellowship.}\\
MIT\\
\texttt{ludwigs@mit.edu}
}

\begin{document}

\maketitle

\setcounter{page}{0}

\thispagestyle{empty}

\begin{abstract}
We study the problem of robustly learning multi-dimensional histograms. 
A $d$-dimensional function $h: D \to \R$ is called a $k$-histogram if there exists a partition of the 
domain $D \subseteq \R^d$ into $k$ axis-aligned rectangles such that $h$ is constant within each such rectangle.
Let $f: D \to \R$ be a $d$-dimensional probability density function 
and suppose that $f$ is $\mathrm{OPT}$-close, in $L_1$-distance, 
to an unknown $k$-histogram (with unknown partition). Our goal is to output a hypothesis
that is $O(\mathrm{OPT}) + \epsilon$ close to $f$, in $L_1$-distance. We give an algorithm for this learning 
problem that uses  $n = \tilde{O}_d(k/\eps^2)$ samples and runs in time $\tilde{O}_d(n)$.
For any fixed dimension, our algorithm has optimal sample complexity, up to logarithmic factors,
and runs in near-linear time. Prior to our work, the time complexity of the $d=1$ case was well-understood, 
but significant gaps in our understanding remained even for $d=2$.
\end{abstract}

\newpage

\section{Introduction}  \label{sec:intro}


{\em Density Estimation} or {\em Distribution Learning} refers to the following unsupervised learning task:
Given i.i.d.\ samples from an unknown target probability distribution, output a hypothesis that is a good approximation to
the target distribution with high probability. Density estimation is a classical and paradigmatic statistical problem 
with a history of more than a century, starting with~\cite{Pearson} 
(see, e.g.,~\cite{BBBB:72, DG85, Silverman:86,Scott:92,DL:01} for textbook introductions).
Despite this long and rich history, core computational aspects of density estimation are wide-open in a variety of settings.
Starting with the pioneering work of~\cite{KMR+:94short}, computer scientists have been working on this broad fundamental question for more than two decades.

The recent distribution learning literature usually studies \emph{structured} settings 
in which the target distribution belongs to a given distribution family ${\cal D}$ 
or is well-approximated by a member of this family with respect to a global loss function.
The complexity of distribution learning often depends heavily on the structure of the underlying family.
The performance of a distribution learning algorithm is typically evaluated by the following criteria:
\begin{itemize}
\item {\em Sample Complexity:} For a given error tolerance, the algorithm should require a small number of samples,
ideally matching the information-theoretic minimum.
\item {\em Computational Complexity:}  The algorithm should run in time polynomial (or, ideally, linear)
in the number of samples provided as input.

\item {\em Robustness:} The algorithm should provide error guarantees under model misspecification,
i.e., even if the target distribution does not belong in the target family ${\cal D}$.
The goal here is to be competitive with the best approximation of the unknown distribution
by any distribution in the family ${\cal D}$.
\end{itemize}

There are two main strands of research in distribution learning.
The first one concerns the learnability of {\em high-dimensional parametric} 
distribution families, e.g., mixtures of Gaussians. The sample complexity of learning parametric families
is typically polynomial in the dimension and the goal is to design computationally efficient algorithms.

The second research strand --- which is the focus of this paper ---
studies the learnability of {\em low-dimensional nonparametric} distribution families under various
assumptions on the shape of the underlying density. There has been a long line of work on this strand
within statistics since the 1950s and, more recently, in theoretical computer science. 
The reader is referred to~\cite{BBBB:72} for a summary of the early work and to~\cite{GJ:14} 
for a recent book on the subject. 
The majority of this literature has studied the univariate (one-dimensional) setting which is by now fairly well-understood
for a wide range of distributions. On the other hand, the {\em multivariate} setting and specifically 
the regime of {\em fixed dimension} is significantly more challenging and 
poorly understood for many natural distribution families.

\subsection{Our Results: Learning Multivariate Histograms}  \label{ssec:results}

In this work, we study the problem of density estimation for the family of histogram distributions on $d$-dimensional domains. 
Throughout this paper, let $[m] = \{1, \ldots, m\}$ denote an ordered discrete domain of size $m$.
A distribution on $[0, 1]^d$ or $[m]^d$ with probability density function $h$ 
is a $k$-histogram if there exists a partition of the domain 
into $k$ axis-aligned hyper-rectangles $R_1, \ldots, R_k$ such that $h$ is constant 
within each of the $R_i$'s.

Histograms constitute one of the most basic nonparametric distribution families.
The algorithmic difficulty in learning such distributions lies in the fact 
that the location and size of these rectangles is unknown to the algorithm. 
Histograms have been extensively studied in statistics and computer science. 
Many methods have been proposed to estimate histogram distributions
~\cite{Scott79, FreedmanD1981, Scott:92, LN96, Devroye2004, WillettN07, Klem09} that
are of a heuristic nature or have a strongly exponential dependence on the dimension. 
In the database community, histograms~\cite{JPK+98,CMN98,TGIK02,GGI+02, GKS06, ILR12, ADHLS15} 
constitute the most common tool for the succinct approximation of data.

The time complexity of learning {\em univariate}
histograms is well-understood: prior work~\cite{CDSS13, CDSS14, CDSS14b, ADLS17}
gives sample-optimal learning algorithms with near-linear running time.
\new{Perhaps surprisingly, no nearly-linear time learning algorithm is known for arbitrary histograms 
even in {\em two} dimensions.
Motivated by this gap in our understanding}, we study the following question: 
\begin{center}
{\em Is there a computationally and statistically efficient algorithm 
\\ to learn arbitrary histograms on $\R^d$, up to $\ell_1$ distance $\eps$?}
\end{center}

\noindent
\new{
Our main result answers this question in the affirmative for any constant dimension: 
}

\begin{theorem}[informal, see Theorem \ref{thm:adaptive-main}]
Fix $\eps > 0$ and $k \in \mathbb{Z}_+$.
Let $f$ be an arbitrary distribution over $[m]^d$ or $[0, 1]^d$.
There is an algorithm which draws $n = \tilde{O}_d(k/\eps^2)$ samples from $f$, runs in time
$\tilde{O}_d(n)$, and outputs a hypothesis $h$ that with high probability satisfies
$\| f - h \|_1 \; \leq \; O(\OPT_k) + \eps$, where 
$\OPT_k = \min_{h'} \| f - h' \|_1$ is the best $\ell_1$-distance achievable by any $k$-histogram.
\end{theorem}

\noindent \new{It is well-known (see, e.g., \cite{ADLS17}) that $\Omega (k / \eps^2)$ samples are 
necessary for any histogram learning algorithm, even for $d=1$. Hence, for any fixed dimension $d$,
our algorithm is sample near-optimal (within logarithmic factors) and runs in sample nearly-linear time. 
Even for $d=2$ and $\OPT_k = 0$, no non-trivial algorithm was previously known for this problem.

A few additional remarks are in order. First, we would like to stress that the focus 
of our work is on the case where the parameters $m, k$ are {\em much larger} than the dimension $d$, 
i.e., $m, k \gg d$. For example, this condition is automatically satisfied when 
$d$ is bounded from above by a fixed constant. This is arguably the most natural setting 
for several applications of multidimensional histograms.
Second, our proof establishes that the hidden multiplicative constant in the $O(\OPT_k)$ 
of the RHS is at most $11$. While we do not know the value of the optimal constant, 
a lower bound of $2$ is known even in one dimension~\cite{CDSS14b}.

Third, the dependence on $d$ in the sample complexity of our algorithm is (weakly) exponential.
Such a dependence in the sample size is not necessary. Standard information-theoretic
arguments give that $\tilde{O}(k d/\eps^2)$ samples suffice --- albeit with a 
$(1/\eps)^{\Omega(k d)}$ time learning algorithm,
which is clearly unacceptable even in one dimension. 
Obtaining a learning algorithm with running time $\poly(d, k, 1/\eps)$ is left
as a challenging open problem. As observed in~\cite{DDS15}, the existence of such an algorithm may be unlikely
as it would imply a $\poly(d, k, 1/\eps)$ time algorithm for PAC learning $k$-leaf decision trees over $\{0, 1\}^d$.}




\new{As a corollary of our algorithmic techniques, we also obtain an efficient ``semi-proper''\footnote{We call our algorithm semi-proper because it produces a hypothesis that is also a histogram but with more than $k$ pieces. For our algorithm, the increase in the number of histogram pieces is a polylogarithmic factor.} learning algorithm for 
discrete histograms with respect to the $\ell_2$-distance. Specifically, we show:}
\begin{theorem}[informal, see Theorem \ref{thm:l2-main}]
\label{thm:l2-informal}
Fix $\eps > 0$ and $k, m, d \in \mathbb{Z}_+$.
Let $f: [m]^d \to \R$ be an arbitrary distribution.
There is an algorithm which draws $n = O(1 / \eps)$ samples from $f$, runs in $O_d(n \log^2 n)$ time,
and outputs an $O_d(k \log^{d + 1} 1 / \eps)$-histogram $h$ so that with high probability
$\| f - h \|_2^2 \leq 2 \cdot \OPT_k + \eps$,
where $\OPT_k = \min_{h'} \| f - h' \|_2^2$ is the best $\ell_2$-squared error achievable by any $k$-histogram.
\end{theorem}

\new{
It is a folklore fact (see, e.g.,~\cite{ADHLS15}) that $\Theta(1/\eps)$ samples are necessary 
and sufficient for this problem and that the empirical distribution is an accurate hypothesis. 
Our algorithm is sample-optimal, runs in near-linear time for constant dimension $d$, and importantly 
provides a succinct ``semi-proper'' hypothesis distribution. 
Succinct data representations by multivariate histograms are well-motivated in several data analysis 
applications in databases, where randomness is used to sub-sample a large dataset~\cite{CGHJ12}.}

\subsection{Our Techniques and Comparison to Prior Work}

\new{
In this section, we provide an overview of our techniques in tandem with a comparison to prior work.
Standard metric entropy arguments (see, e.g.,~\cite{DL:01})
yield an inefficient method that uses $\tilde{O}(kd/\eps^2)$ samples and runs in time $(1/\eps)^{\Omega(kd)}$.
To avoid the exponential dependence on $k$ in the runtime, one can first partition the domain into
$\poly(k/\eps)^{\Theta(d)}$ ``light'' rectangles and then learn the induced probability distribution on these
rectangles. This naive learning algorithm inherently incurs sample complexity and running time 
of $\poly(k/\eps)^{\Theta(d)}$, which makes it unsatisfying even for $2$ dimensions.
}

Our algorithms rely on two main ideas.
The first ingredient is a \emph{greedy splitting} scheme that 
enables us to approximate multi-dimensional histograms efficiently.
In contrast to one-dimensional histograms, the partitions induced by multi-dimensional histograms 
are too complicated for a direct dynamic programming approach.
Similarly, the approximate {\em iterative merging} strategy analyzed in \cite{ADLS17} 
does not seem to generalize to the multi-dimensional setting: merging 
two adjacent rectangles does not necessarily yield another rectangle (as opposed to adjacent intervals).
We circumvent the difficulties introduced by the complex structure 
of arbitrary histogram partitions by going through hierarchical histograms, 
which yield a more structured space of partitions that is amenable to efficient algorithms.
\new{\cite{WillettN07} used a related decomposition to learn smooth classes of continuous 
densities. First, we note that our algorithm and its analysis are significantly different from theirs. 
Second,~\cite{WillettN07} do not obtain a near-linear time algorithm 
even in one dimension. In the univariate setting, \cite{DiakonikolasG0N17} used a similar algorithm 
to learn discrete distributions in the distributed setting with respect to the $\ell_2$-norm.}

Hierarchical histograms have appeared before in histogram approximation, 
especially in the setting of wavelet-based approaches (for instance, see \cite{GKMS01,GGI+02}) 
but also in approximate dynamic programs such as \cite{MPS99}. However, these approaches do not handle 
the $\ell_1$-setting that is standard in distribution learning.
Instead, we propose a top-down splitting algorithm that expands leaf nodes in a growing hierarchical histogram according to a special error metric that we call the \emph{$\hier$-distance}.
The $\hier$-distance is closely related to VC theory and allows us to make good splitting decisions not only for the empirical distribution but also for the unknown distribtion we aim to recover.

The basic version of our greedy splitting scheme relies on hierarchical partitions of the distribution domain $[m]^d$, which incurs a logarithmic dependence on the domain size and does not apply to the continuous setting.
The second ingredient in our paper is an \emph{adaptive} variant of our splitting algorithm.
This variant makes splitting decisions not on the dyadic boundaries of a data-independent hierarchical partition, but instead relies on the empirical distribution to build a data-dependent grid of coordinate points.
By restricting our attention to the relevant coordinates, we can remove the logarithmic dependence on $m$ and also apply our algorithm to distributions defined on $[0,1]^d$.
The adaptive approach requires a more careful analysis of our splitting algorithms and relies on the notion of a \emph{partial} hierarchical histogram.
In a partial hierarchical histogram, each partition can ``shrink'' to the bounding box of the samples in the partition, leaving a region on which the partition assigns value 0.
Our final adaptive splitting algorithm runs in time that is nearly-linear in the number of samples with no dependence on the domain size.
This is in contrast to prior wavelet-based approaches, which usually have a logarithmic dependence on the domain size and often process the entire domain $[m]^d$ as opposed to only the non-zero sample points.

\begin{figure}
  \centering
  \begin{subfigure}[t]{0.3\textwidth}
    \includegraphics[width=\textwidth]{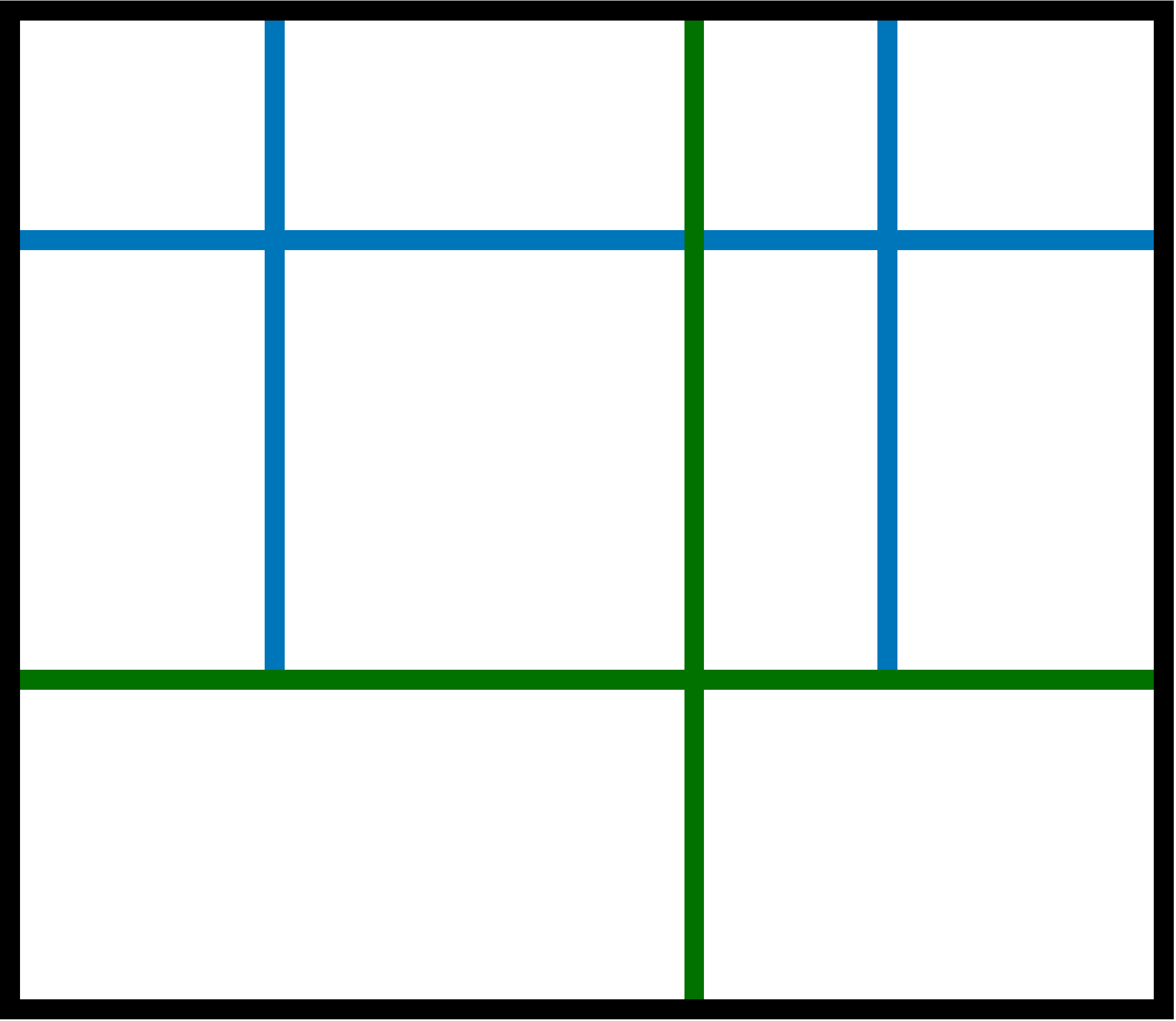}
    \caption{Current partition}
  \end{subfigure}
  \hfill
  \begin{subfigure}[t]{0.3\textwidth}
    \includegraphics[width=\textwidth]{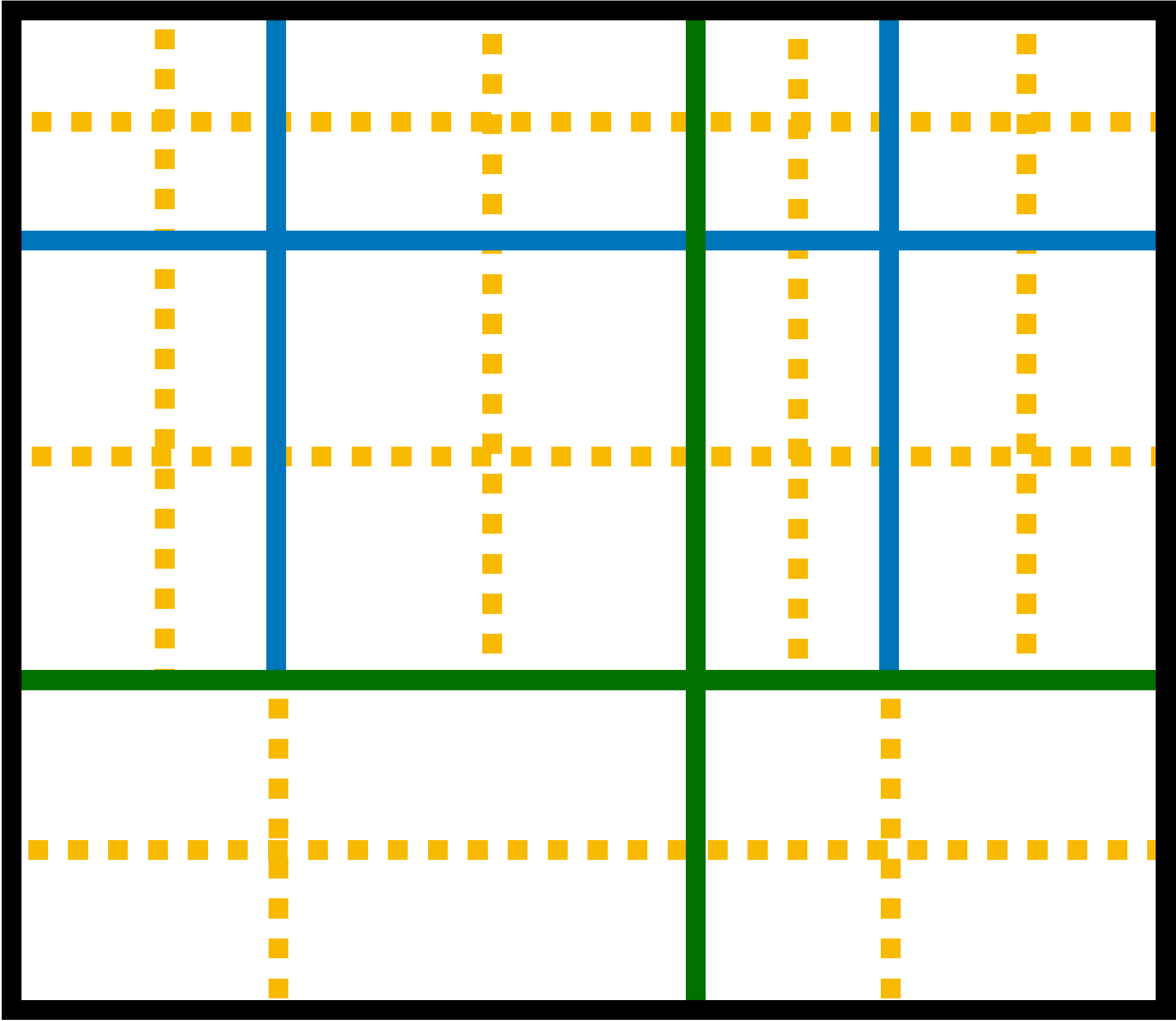}
    \caption{Possible leaf refinements}
  \end{subfigure}
  \hfill
  \begin{subfigure}[t]{0.3\textwidth}
    \includegraphics[width=\textwidth]{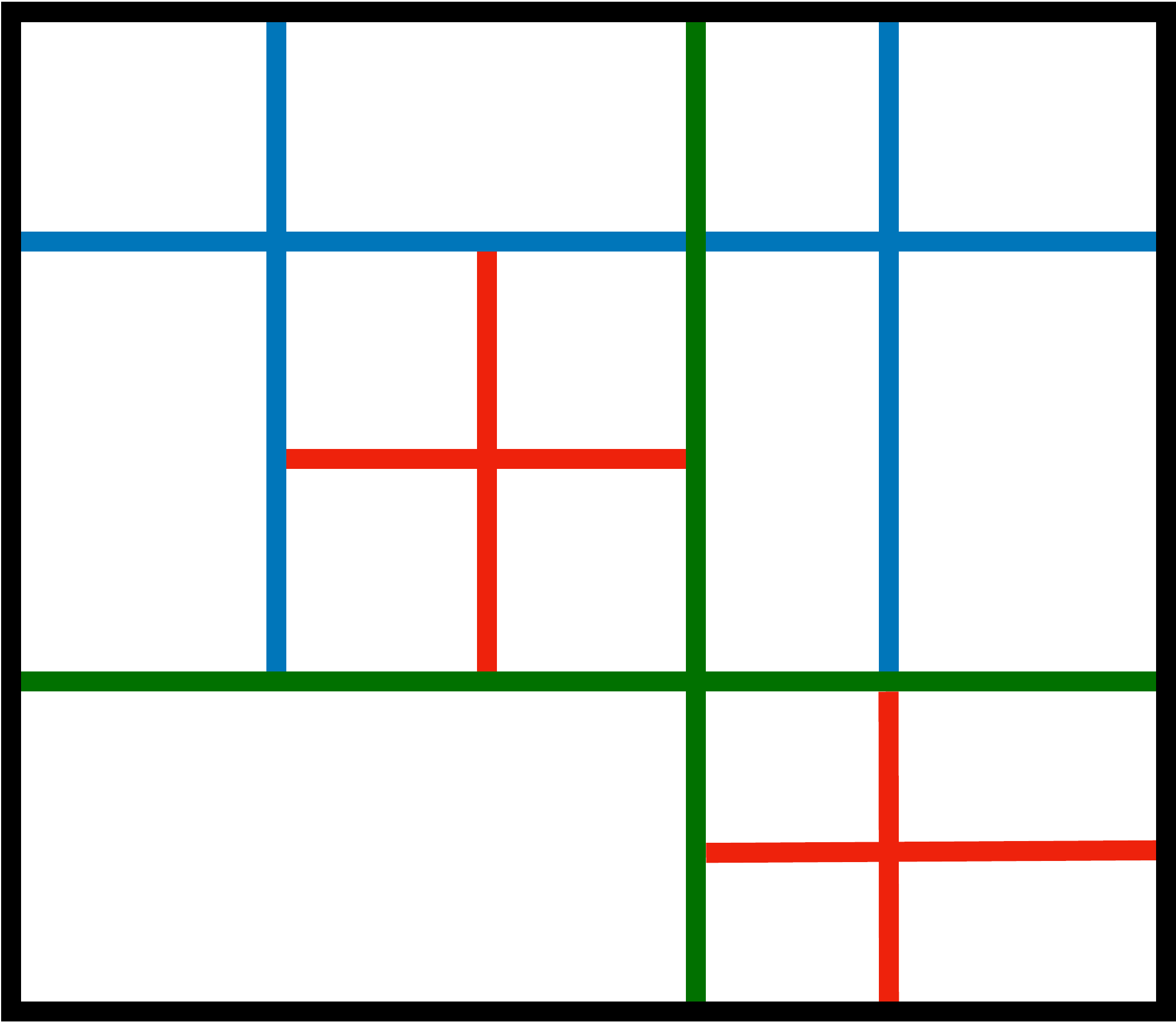}
    \caption{Next partition}
  \end{subfigure}
  \caption{One iteration of our adaptive partitioning scheme.
  The left sub-figure (a) displays the hierarchical partitioning of $\R^2$ after two iterations of the algorithm.
  It is derived from two levels of splits: first the green split, then the blue splits.
  The center sub-figure (b) shows the candidate leaf splits that the algorithm considers as next refinements.
The algorithm chooses the splits that most reduce a certain error metric (see the right sub-figure (c)). }
\end{figure}

%

\section{Preliminaries}
We define the $\ell_p$-norm of a measurable function 
$f: [m]^d \to \R$ or $f: [0, 1]^d \to \R$ to be $\| f \|_p = \left( \sum_{x \in [m]^d} |f(x)|^p \right)^{1 / p}$ 
or $\| f \|_p = \left( \int |f (x)|^p dx \right)^{1 / p}$, for $[m]^d$ and $[0, 1]^d$ respectively.
For any subset $R \subseteq [m]^d \to \R$ (similarly for $[0, 1]^d$), we let $\| f \|_{p, R}$ be the $\ell_p$-norm of $f$ restricted to $R$.
Given $X_1, \ldots, X_n$ samples from a distribution $f$ supported over $[m]^d$ (resp $[0, 1]^d$), 
we let the empirical distribution induced by these samples be 
$\fhat = \frac{1}{n} \sum_{i = 1}^n \delta_{X_i}$, where $\delta_X$ 
is the delta distribution supported at $X$.

\subsection{Histograms and Problem Definition}
We first define the notion of histograms.
Throughout this paper, we will assume w.l.o.g. that $m$ is a power of $2$.
\begin{definition}
A distribution $h: \new{A} \to \R$, \new{where $A$ is either $[m]^d$, for $m \in \Z_+$, or $[0, 1]^d$,}
is called a \emph{$k$-histogram} if there exists a partition of \new{$A$} 
into $k$ axis aligned rectangles $R_1, \ldots, R_k$ so that $h$ is constant on $R_i$, for all $i = 1,\ldots, k$.
We let $\calH_k$ denote the set of $k$-histograms.
\end{definition}

\noindent
We now can state the formal problem:
\paragraph{Problem Statement} Given $0< \eps, \delta <1$ and independent samples 
from some distribution $f: \new{A} \to \R$ \new{where $A$ is either $[m]^d$ or $[0, 1]^d$}, return $\hhat$ so that 
with probability $1 - \delta$, we have $\| \hhat - f \|_1 \leq C \cdot \OPT_k + \eps$, 
where $C$ is an absolute constant and
\[
\OPT_k = \min_{h \in \calH_k} \| h - f \|_1 \; .
\]

\noindent
We will also crucially make use of the following definition throughout the paper:
\begin{definition}
Let $g$ be any function over $[m]^d$ (resp. $[0, 1]^d$).
For any set $R \subseteq [m]^d$ (resp. $[0, 1]^d$), define the flattening of $g$ over $R$, 
denoted $\flatten{g}_R$, to be the constant function on $R$ which takes on value $g(R) / |R|$ at each point in $R$.
For any collection of disjoint sets $\mathcal{R}$, define the flattening of $g$ over $\mathcal{R}$, 
denoted $\flatten{g}_R$, to be the function which is equal to the flattening of $g$ on each set $R \in \mathcal{R}$.
\end{definition}

\subsection{Hierarchical Histograms}
We also require the notion of a \emph{hierarchical} histogram, which is a histogram that respects a fixed dyadic partition.
Formally:
\begin{definition}
Given a grid $\calG = P_1 \times P_2 \times \ldots \times P_d$, where each $P_i$ is 
a collection of elements $x^{(i)}_1 \leq x^{(i)}_2 \leq \ldots \leq x^{(i)}_M$ in $[m]$ (resp. $[0, 1]$) and $M$ is a power of $2$, 
\emph{the level-$\ell$ rectangles} induced by $\calG$, denoted $\rect_\ell$, is defined to be
\[
\rect_\ell = \left\{ \otimes_{i = 1}^d [x^{(i)}_{2^\ell j_i + 1}, x^{(i)}_{2^{\ell} (j_{i} + 1) }] : j_i \in \{0, \ldots, M / 2^\ell - 1\} \right\} \; .
\]
Moreover, the \emph{dyadic decomposition} of $\calG$, denoted $\hier = \hier (\calG)$, is defined to be $\hier = \bigcup_{\ell = 1}^{\log M} \rect_\ell$.
For any $k \geq 1$, and a dyadic decomposition $\hier$ of a grid $\calG$ we let $\hier_k$ denote all disjoint unions of at most $k$ rectangles from $\hier$.
\end{definition}

For instance, if the domain is $[m]^d$ and each $P_i = [m]$, then the induced dyadic decomposition is 
simply the set of squares $R$ with side-length $2^\ell$ for some $\ell = 1, \ldots, \log m$ 
and whose rightmost vertices are at a power of $2$.
In general, any dyadic decomposition induces a natural tree structure, which we will utilize throughout the paper.
We can now define our notion of a hierarchical histogram:

\begin{definition}
We say a $k$-histogram $f: \new{A} \to \R$ \new{where $A$ is either $[m]^d$ or $[0, 1]^d$} is hierarchical with respect to a grid $\calG$ 
if there exists a partition of $\new{A}$ into rectangles $R_1, \ldots, R_k \in \hier(\calG)$ 
so that $f$ is constant on each $R_i$. If $\calG$ is understood, we say $f$ is hierarchical for short.
\end{definition}

\noindent
We have the following simple lemma, which says that we may assume w.l.o.g. that the histogram is hierarchical, with some loss:
\begin{lemma}
\label{lem:arbitrary-to-hierarchical}
Fix a grid $\calG$ with side length $M$.
Let $f: \new{A} \to \R$ \new{where $A$ is either $[m]^d$ or $[0, 1]^d$} be a $k$-histogram so that it is constant on $R_1, \ldots, R_k$, 
where every vertex of every rectangle lies on $\calG$.
Then $f$ is a $k \log^d M$-hierarchical histogram.
\end{lemma}
\begin{proof}
For simplicity of exposition we will show this assuming $\calG = [m]^d$, so the side length is equal to $m$.
\new{The same proof easily extends to general grids, and so we omit the details for conciseness}.
It suffices to show that any function which is supported within an axis-aligned rectangle $R$ 
and which is constant within this rectangle can be represented as a $\log^d m$-hierarchical histogram.
Let $R = [a_1, b_1] \times [a_2, b_2] \times \ldots \times [a_d, b_d]$.
Each interval $[a_i, b_i]$ can be written as a union of at most $\log m$ 
disjoint dyadic intervals $\setI_i$, so $R$ can be decomposed 
as the disjoint union of all rectangles $R = \otimes_{i = 1}^d I_i$ where each $I_i$ ranges over all intervals in $\setI_i$.
By inspection, this requires $\log^d m$ pieces.
\end{proof}

\noindent Thus, we lose $\log^d M$ factors going 
from arbitrary histograms to hierarchical histograms, 
where $M$ is the side length of our grid.

\subsection{VC Theory}
We now need the following classical definition of $\VC$-dimension:
\begin{definition}[VC dimension]
A collection of sets $\A$ is said to \emph{shatter} a set $S$ if for all $S' \subseteq S$, there is an $A \in \A$ so that $A \cap S = S'$.
The VC dimension of $\A$, denoted $\VC(\A)$, is the largest $n$ so that there exists a $S$ with $|S| = n$ so that $\A$ shatters $S$.
\end{definition}

\noindent
For any collection $\mathcal{A}$ of measurable subsets of $[m]^d$ or $[0, 1]^d$, 
define the $\mathcal{A}$-norm, denoted $\| \cdot \|_\mathcal{A}$, on measurable real-valued functions on $\R^d$ to be
$$\| f \|_\A = \sup_{A \in \A} |f(A)| \;.$$
For any measurable subset $R$ of $[m]^d$ or $[0, 1]$, 
we also define $\| \cdot \|_{\mathcal{A}, R}$ to be the $\A$-norm of the function restricted to $R$.
We now need the following form of the VC theorem, which follows by combining a classical form of the VC theorem 
along with standard uniform deviation arguments (e.g., McDiarmid's inequality): 
\begin{theorem}[c.f. Devroye \& Lugosi Theorems 4.3 and 3.2, \cite{CDSS14}, Theorem 2.2]
\label{thm:vc}
Let $f: [m]^d \to \R$ be a distribution, and let $\fhat_n$ denote the empirical distribution after $n$ independent draws from $f$.
Then, for all $\delta > 0$,
$
\Pr \left[ \| f - \fhat_n \|_\A \geq \sqrt{\frac{\VC(\A) + \log 1 / \delta}{n}} \right] \leq \delta
$.
\end{theorem}

\section{Learning Histograms in $\ell_1$-Distance}
\label{sec:l1}
We now consider the question of histogram approximation in $\ell_1$.
The main difficulty in learning in $\ell_1$ (as opposed to, say, in $\ell_2$), is that the statistical and algorithmic questions do not nicely decouple.
This is because the empirical distribution is not close to the true distribution until many samples are taken.
Instead, we will have to consider a different algorithmic objective inspired by VC theory.

\subsection{Computing $\hier$-distance and fitting in $\hier$-distance}
Recall that $\hier$ is defined to be the set of dyadic rectangles over $[m]^d$ (resp. $[0, 1]^d$).
By the theory developed above, this naturally induces a metric on functions from $[m]^d$ (resp. $[0, 1]^d$) to $\R$.
In this section, we show that computing and fitting with respect to $\hier$ distance can be done in nearly input-sparsity time.
Throughout this section, fix any grid $\calG$ of side length $M$ for $[m]^d$ or $[0, 1]^d$.
For any $a \in \R$ and $R \in \hier$, let $\phi_{a, R}: R \to \R$ denote the function on $R$ which is constantly $a$.
We show:

\begin{lemma}
\label{lem:compute-A1}
Given an empirical distribution $\fhat$, a rectangle $R \in \mathcal{D}$ so that $\fhat$ 
is supported on $s$ points in $R$, and a $a \in \R$, there is an algorithm \textsc{ComputeD1} 
that runs in time $O(2^d s \log M)$ and outputs $\| \fhat - \phi_{a, R} \|_{\hier, R}$ 
 together with a rectangle in $\hier$ achieving this maximum.
\end{lemma}
\noindent
For conciseness we defer the proof of Lemma \ref{lem:compute-A1} to Appendix \ref{sec:l1-app}.
We now show that as a simple consequence of this, we can (approximately) 
find the constant fit to $\fhat$ on any rectangle $R$ in $\| \cdot \|_{\hier, R}$-norm in nearly linear time:
\begin{corollary}
Given $\gamma > 0$, an empirical distribution $\fhat$, a rectangle $R \in \mathcal{D}$ 
so that $\fhat$ is supported on $s$ points in $R$, there is an algorithm \textsc{FitD1} which outputs an $a \in \R$ so that 
$
\| \fhat - \phi_{a, R} \|_{\hier, R} \leq \min_{a' \in \R} \| \fhat - \phi_{a', R} \|_{\hier, R} + \gamma
$
in time $\Otilde(\poly(d) \cdot s \log M \log 1 / \gamma)$.
\end{corollary}

The algorithm is simple: we reduce the optimization problem 
with binary searching over feasibility problems, then solve each feasibility problem 
using \textsc{ComputeD1} as a separation oracle.
The details are subsumed by the calculations for Theorem~31 of \cite{ADLS17}, so we omit them.

For simplicity, we shall assume for the rest of the paper that \textsc{FitD1} produces an exact fit in $\A_1$-distance.
Because the dependence on $\gamma$ in the runtime is logarithmic, 
it is not hard to see that by taking $\gamma = \poly (k, 1 / \eps, \log 1 / \delta)^d$ in the remainder, 
we only increase the approximation errors throughout by at most additive $\eps$ factors, 
and this keeps the runtime unchanged, up to log factors.

\subsection{The Greedy Splitting Algorithm for $\ell_1$-Distance}
In this section, we give an efficient algorithm for constructing 
hierarchical histograms for fitting a known empirical distribution 
in the norm induced by the hierarchical decomposition.
Throughout this section, fix a grid $\calG$ with side length $M$ 
over either $[m]^d$ or $[0,1]^d$, and let $\hier = \hier(\calG)$ 
be the induced dyadic decomposition.

We will prove that our output, despite being a hierarchical histogram, is actually competitive 
with the best error achievable by a slightly more general class of functions, 
which we call partial hierarchical histograms.
Formally:
\begin{definition}
A partial $k$-histogram $h: [m]^d \to \R$ (or $h: [0, 1]^d \to \R$) 
is a distribution \new{satisfying the following:} there exist $k$ disjoint rectangles $R_1, \ldots, R_k$ 
such that $h$ is supported on $\bigcup_{i = 1}^k R_i$, and on each $R_i$, $h$ is constant.
We say that $h$ is a partial $k$-hierarchical histogram 
with respect to a grid $\calG$ if in addition we have $R_i \in \hier (\calG)$ for all $i$.
\end{definition}

\noindent
Our main algorithmic theorem is:
\begin{theorem}
\label{thm:splitting}
Fix $k \in \mathbb{Z}_+$, and let $\gamma > 0$ be a tuning parameter.
Let $\fhat$ be an empirical distribution on $s$ points.
There is an algorithm \textsc{GreedySplit} which outputs a $O((1 + \xi) 2^d k \log M)$-hierarchical histogram $h$ so that
$\| h - \fhat \|_{\hier_k} \leq \left( 3 + \frac{6}{\xi^2} \right) \cdot \tOPT_{\hier, k} + \eps$,
where $\tOPT_{\hier, k} = \min_{h} \| h - \fhat \|_{\blue{\hier_k}}$,
where the minimum is taken over all partial hierarchical $k$-histograms $h$.
Moreover, the algorithm runs in time $\Otilde (2^d s \log^2 M)$.
\end{theorem}
\noindent
Our algorithm, given formally in Algorithm \ref{alg:hist-l1}, is quite simple.
We construct a tree of nested dyadic rectangles.
Initially, this tree contains only $[m]^d$.
Iteratively, we find the leaves of this tree with largest 
$\blue{\hier}$-distance error to $g$, and we split these into all of its children, and we repeat this for $\log M$ iterations.
At the end, we return the flattening of $g$ over all the leaves in the final tree.
For conciseness, we defer the proof of Theorem \ref{thm:splitting} to Appendix \ref{sec:l1-app}.

\begin{algorithm}[htb]
\begin{algorithmic}[1]
\Function{GreedySplit}{$\fhat, \hier, \xi$}
\State Let $\calT$ be a subtree of the hierarchical tree, initially containing only the root.
\For{$\ell = 1, \ldots, \log M$}
	\For{each leaf $R \in \calT$}
		\State Let $a_R = \textsc{FitD1} (\fhat, \hier, R)$
		\State Let $e_R = \textsc{ComputeD1} (\fhat, R, a_R)$
	\EndFor
	\State Let $\setJ$ be the set of $(1 + \xi) k$ leaves $R \in \calT$ with largest $e_R$. \label{line:J}
	\For{each $R \in \setJ$}
		\If{$R$ can be subdivided in $\hier$ and $e_R > 0$}
			\State Add all children of $R$ to $\calT$
		\EndIf
	\EndFor
\EndFor
\State \textbf{return} The function which is constantly $a_R$ for every leaf $R$ of $\calT$
\EndFunction
\end{algorithmic}
\caption{A greedy splitting algorithm for learning hierarchical histograms in $\hier_k$-distance}
\label{alg:hist-l1}
\end{algorithm}

\subsection{Warm-up: an Algorithm for Hierarchical Histograms on $[m]^d$}

In this section, we will take $\calG = [m]^d$.
Assume for simplicity that $m$ is a power of $2$.
Then, we may take $\hier$ to be the dyadic partition of $[m]^d$, i.e., $\hier = \{ \rect_i \}_{i = 1}^{\log m}$ where 
\[
\rect_i = \left\{ \otimes_{i = 1}^d [j_i m 2^{-i} + 1, (j_i + 1) m 2^{-i}] : j_i \in \{0, \ldots, 2^{i} - 1\} \right\}
\]
are all rectangles on a $m 2^{-i}$-spaced grid.
The following are standard facts from VC theory and we defer their proof to Appendix \ref{sec:l1-app}.

\begin{corollary}
\label{cor:difference-of-histograms}
Let $f, g$ be two $k$-hierarchical histograms.
Then $2 \| f - g \|_{\hier_{2k}} = \| f - g \|_1$.
\end{corollary}

\begin{corollary}
\label{cor:VC-hierk}
For all $k, d \geq 1$, we have $\VC (\hier_k) = O(k d)$.
\end{corollary}

\noindent
These corollaries together imply:
\begin{corollary}
\label{cor:hierarchical-main}
Fix $\eps, \delta > 0$, and let $\gamma > 0$.
Let $f: [m]^d \to \R$ be an arbitrary distribution.
Then, the algorithm $\textsc{GreedySplit}(\fhat, \hier ([m]^d), \xi)$, given $\fhat = \fhat_n$ which is the empirical distribution of $f$ after
$
n = \Omega \left( \frac{(1 + \xi) 2^d k \log^{d + 1} m + \log 1 / \delta}{\eps^2} \right)
$
samples, outputs a $(1 + \xi) 2^d d k \log^{d  +1} m$-hierarchical histogram $h$ so that with probability $1 - \delta$, we have
$
\| f - h \|_1 \leq \left( 3 + \frac{6}{\xi^2} \right) \cdot \OPT_k + \eps$.
Moreover, this algorithm runs in time $O(2^d n \log^2 m)$.
\end{corollary}
\begin{proof}
The bound on the number of pieces and the runtime of the algorithm follow from Lemmas \ref{lem:pieces} and \ref{lem:runtime} immediately.
Thus, it suffices to argue about correctness.
By Lemma \ref{lem:arbitrary-to-hierarchical}, we know that if we let $\OPT'_{k \log^d m}$ 
be the optimal $\ell_1$-error to $f$ achievable by a hierarchical $k \log^d m$-histogram, 
then $\OPT'_{k \log^d m} \leq \OPT_k$.
Let $h^*_{\hier}$ be the hierarchical $k$-histogram which achieves the optimum.

Condition on the event that 
$\| f - \fhat \|_{\hier_\kappa} \leq c \eps$,
where $\kappa = 2 (1 + \xi) 2^d k \log^{d + 1} m$, for some universal constant $c$ sufficiently small.
By Theorem \ref{thm:vc}, this happens with probability $1 - \delta$ if we take 
$
n = \Omega \left( \frac{(1 + \xi) 2^d k \log^{d + 1} m + \log 1 / \delta}{\eps^2} \right)
$
samples.
Then, we have
\begin{align*}
\| \fhat - h^*_\hier \|_{\hier_\kappa} &\leq \| f - \fhat \|_{\hier_\kappa} + \| f - h^*_\hier \|_{\hier_\kappa} \leq \OPT'_{k \log^d m} + c \eps \; .
\end{align*}
Therefore, we have $\tOPT_{k \log^d m} \leq \OPT'_{k \log^d m} + c \eps$.
Combining this with the guarantee from Theorem \ref{thm:splitting} then immediately yields the statement, for $c'$ chosen to be sufficiently small.
\end{proof}

\subsection{General Histograms via Adaptive Gridding}
The framework presented above is very clean, however, it has one major drawback.
Namely, the conversion from arbitrary to hierarchical histograms on the grid $[m]^d$ loses $\log^d m$ factors. 
In particular, these factors prevent the algorithm from being useful when the support size is large or infinite.
In this section, we show that a modification of the techniques presented above can remove these factors.
The algorithm in this section will work even when the support size is infinite.
Throughout the section, we will state our results for $[m]^d$, however, they generalize trivially to $[0, 1]^d$, and we omit the details for simplicity.
Our main result in this section is:
\begin{theorem}
\label{thm:adaptive-main}
Fix $\eps, \delta > 0$, and let $\gamma > 0$.
Let $f: [m]^d \to \R$ be an arbitrary distribution.
There is an algorithm \textsc{AdaptiveGreedySplit}, which, given $n$ independent samples from $f$, where
$
n = O \left(  \frac{(1 + \xi) d 2^d k \log^{d + 2} (k / \eps) + \log 1 / \delta}{\eps^2} \right)
$,
outputs a $O((1 + \xi) d 2^d k \log^{d + 1} (k / \eps))$-hierarchical histogram $h$ so that with probability $1 - \delta$, we have
$\| f - h \|_1 \leq \left( 10 + \frac{12}{\xi^2} \right) \cdot \OPT_k + \eps$.
Moreover, this algorithm runs in time $O(2^d n \log^2 n)$.
\end{theorem}

\paragraph{The VC dimension of (Partial) Histograms}
In this section, we bound the VC dimension of set systems induced by differences between $k$-histograms and partial histograms.
We first need the following fact, which is a direct implication of the respective definitions:
\begin{fact}
\label{fact:difference-of-partials}
Given two $k$-partial histograms $h, g: [m]^d \to \R$, the set $\{x: h(x) > g(x) \}$ is of the form $\bigcup_{i = 1}^{k'} A_i - \bigcup_{j = 1}^{k''} B_j$, for some axis aligned rectangles $A_i, B_j$ so that the $A_i$ are mutually disjoint and $B_j$ are mutually disjoint, and $k', k'' \leq k$.
\end{fact}

\noindent
Motivated by this fact, we let
\[
\A_k = \left\{ \bigcup_{i = 1}^{k'} A_i - \bigcup_{j = 1}^{k''} B_j : ~\mbox{$\{ A_i \}_{i = 1}^{k'}, \{ B_j \}_{j = 1}^{k''}
$ are collections of disjoint rectangles, $k', k'' \leq k$}~ \right\} \; 
\] 
be the set system that captures sign difference between $k$-partial histograms.
By Fact \ref{fact:difference-of-partials}, we have:
\begin{corollary}
For any two $k$-partial histograms $h, g: [m]^d \to \R$ (or over $[0, 1]^d$), we have
$
\| h - g \|_1 = 2 \| h - g \|_{\A_k}.$
\end{corollary}
\noindent
We now require a bound on the VC dimension of $\A_k$, whose proof we defer to the appendix:
\begin{lemma}
\label{lem:VC-histograms}
For all $k \geq 1$, we have $\VC (\A_k) = O(k d \log (k d))$.
\end{lemma}

\noindent
As an immediate corollary of Theorem \ref{thm:vc} and Lemma \ref{lem:VC-histograms}, we have:
\begin{corollary}
\label{cor:VC-histograms}
Fix $\eps, \delta > 0$.
Let $f:[m]^d \to \R$ be an arbitrary distribution.
Let $\fhat = \fhat_n$ be the empirical distribution given $n$ independent samples from $f$, where
$
n = O \left( \frac{k d \log (k d) + \log 1 / \delta}{\eps^2} \right)$.
Then, with probability $1 - \delta$, we have
$\| \fhat - f \|_{\A_k} \leq \eps$.

\end{corollary}
\noindent

\noindent
For the rest of the section, we let $f$ denote the unknown distribution, and we let $\fhat = \fhat_n$ denote the empirical distribution after $n$ draws from $f$, where 
\[
n =  C \frac{(1 + \xi) d 2^d k \log^{d + 2} (k / \eps) + \log 1 / \delta}{\eps^2} \; ,
\]
for some universal constant $C$ sufficiently large.
We let $\calX$ denote the (multi-)set of samples, 
i.e., $\calX = \supp (\fhat)$, and we will, in a slight abuse of notation, let $\hier = \hier (\grid (\calX))$.

We will condition on the event that
\begin{equation}
\label{eq:deterministic-cond}
\| \fhat - f \|_{\A_{\kappa}} \leq c' \eps \; ,
\end{equation}
for some universal constant $c'$ sufficiently small, where $\kappa = C' (1 + \xi) 2^d k \log^{d + 1} (k / \eps)$ 
for some universal constant $C'$ sufficiently large.
Observe that since $\hier_k \subseteq \A_k$, this immediately implies that 
$
\| \fhat - f \|_{\hier_k} \leq c' \eps$.
By Corollary \ref{cor:VC-histograms}, this holds with probability $1 - \delta$ as long as we take at least
\[
n = \Omega \left( \frac{d \kappa \log (d \kappa) + \log 1 / \delta}{\eps^2} \right) \; .
\]
In particular, this holds for our choice of $n$, for $C$ sufficiently large.

\paragraph{Rounding Histograms to Partial Hierarchical Histograms}
Our algorithm is straightforward: we simply grid over all points 
where the samples land, that is, we take the grid to be $\calG = \grid (\calX)$, 
then find the best fit hierarchical histogram with respect to this grid, 
and the norm it induces, using the same algorithm as above.
Our algorithm will then be very similar to the algorithm presented previously, 
with some crucial but subtle changes, however, the analysis requires some additional steps.

In particular, now it is not a priori clear that the optimal histogram 
fit to the true density will have vertices on the grid, and in general, 
it is not too hard to show that it will not.
However, we show that by only losing constant factors in the approximation ratio, 
we may as well assume that it does, with some important caveats.
Specifically, we show that we may approximate the optimal fit $k$-histogram 
to $f$ with a $k$-partial histogram with vertices on the grid.
Formally:
\begin{lemma}
\label{lem:arbitrary-to-hierarchical-rounding}
Fix $\eps > 0$, and assume that (\ref{eq:deterministic-cond}) holds.
Then, there is a $\kappa' = O(k \log^d (k / \eps))$-partial histogram $h_p^*$ that is hierarchical with respect to $\calX$ so that 
\begin{equation}
\label{eq:arbitrary-to-hierarchical}
\| f - h^*_p \|_1 \leq 2 \cdot \OPT_k + 4 c' \eps \; .
\end{equation}
\end{lemma}
We defer the proof of Lemma \ref{lem:arbitrary-to-hierarchical-rounding} to Appendix \ref{sec:l1-app}.
We also need the following lemma, which states that the $\hier_k$-distance 
still captures the $\ell_1$-distance between a partial hierarchical histogram 
and a (regular) hierarchical histogram.
\begin{lemma}
\label{lem:partial-minus-full}
Fix a grid $\calG$, and let $h$ be a partial hierarchical $k$-histogram, 
and let $g$ be a hierarchical $k$-histogram, both with respect to $\calG$.
Then, $\{x: h(x) > g(x)\} \in \hier_k (\calG)$.
In particular, this implies that $\| h - g \|_1 = 2 \| h - g \|_{\hier_{2k}}$.
\end{lemma}
For conciseness we defer the proof to Appendix \ref{sec:l1-app}.

\paragraph{Putting Everything Together}
We now have the tools to prove Theorem \ref{thm:adaptive-main}.
The algorithm is fairly simple: we take the grid induced by our samples, 
and run \textsc{GreedySplit} on this grid on the empirical distribution.
The formal pseudocode is given in Algorithm \ref{alg:ada-hist-l1} in Appendix \ref{sec:l1-app}.

\begin{proof}[Proof of Theorem \ref{thm:adaptive-main}]
The runtime guarantee and the guarantee on the number of pieces easily follow 
from Theorem \ref{thm:splitting}.
Thus, it suffices to prove correctness.
Let $\calX$ denote the set of samples, let $\calG = \grid (\calX)$, and $\hier = \hier(\calG)$.
Recall $h$ is the output of our algorithm, and let $h^*$ be the optimal $k$-histogram fit to $f$ in $\ell_1$.
By Lemma \ref{lem:arbitrary-to-hierarchical-rounding}, we know that there 
is some partial hierarchical $O(k \log^d (k / \eps))$-histogram $h^*_p$ so that
\begin{equation}
\label{eq:OPT-partial}
\| f - h^*_p \|_1 \leq 2 \OPT_k + 2 c' \eps \; .
\end{equation}
In particular, this implies that 
\begin{align*}
\| \fhat - h^*_p \|_{\hier_{\kappa}} &\leq \| \fhat - f \|_{\hier_{\kappa}} + \| f - h^*_p \|_{\hier_{\kappa}} 
\\
&\leq \eps +  2 \OPT_k + 2 c' \eps \\
&= 2 \OPT_k + 3 c' \eps \; ,
\end{align*}
and hence $\tOPT_{\hier, k} \leq 2 \OPT_k + 3 c' \eps$.
We thus have 
\begin{align*}
\| f - h \|_1 &\leq \| f - h^*_p \|_1 + \| h^*_p - h \|_1 \\
&\stackrel{(a)}{\leq} 2 \cdot \OPT_k + 2 c' \eps + \| h^*_p - h \|_{\hier_{\kappa}} \\
&\stackrel{(b)}{\leq} 2 \cdot \OPT_k + 2 c' \eps + \| h^*_p - \fhat \|_{\hier_{\kappa}} + \| \fhat - h \|_{\hier_{\kappa}} \\
&\stackrel{(c)}{\leq} 2 \cdot \OPT_k + 2 c' \eps + \| h^*_p - f \|_1 + \| f - \fhat \|_{\hier_{\kappa}} + \| \fhat - h \|_{\hier_{\kappa}} \\
&\stackrel{(d)}{\leq} 4 \OPT_k + 5 c' \eps + \left( 3 + \frac{6}{\xi^2} \right) \tOPT_{\hier, k} \\
&\leq \left( 10 + \frac{12}{\xi^2} \right) \OPT_k + O(c' \eps) \; ,
\end{align*}
where (a) follows from a triangle inequality, (\ref{eq:OPT-partial}), 
and Lemma \ref{lem:partial-minus-full}, (b) and (c) follow from the triangle inequality, 
and (d) follows from (\ref{eq:OPT-partial}), (\ref{eq:deterministic-cond}) and Theorem \ref{thm:splitting}.
By choosing $c'$ sufficiently small, this completes the proof.
\end{proof}


\bibliographystyle{alpha}
\bibliography{allrefs}


\appendix

\section{Omitted Proofs from Section \ref{sec:l1}}
\label{sec:l1-app}

\subsection{VC Theory for Hierarchical Histograms}
We first characterize exactly the structure of the difference between two hierarchical histograms that respect the dyadic decomposition on $[m]^d$:
\begin{lemma}
\label{lem:difference-of-histograms}
Let $f, g$ be two $k$-hierarchical histograms with respect to a grid $\calG$.
Then $f - g$ is a $2k$-hierarchical histogram.
\end{lemma}
\begin{proof}
Let $R_1, \ldots, R_k$ and $R_1', \ldots, R_k'$ be rectangles in the hierarchical structure so that $f$ is flat on each rectangle $R_i$ and $g$ is flat on each rectangle $R_i'$.
For each pair of rectangles $R_i$ and $R_j'$, we have that either $R_i \subseteq R_j'$ (or vice versa), or $R_i \cap R_j' = \emptyset$.
Thus, if we choose a maximal subset $\mathcal{S}$ of $\{R_1, \ldots, R_k, R_1', \ldots, R_k'\}$ so that (1) there do not exist $R, R' \in \mathcal{S}$ so that $R \subseteq R'$, and moreover, (2) there does not exist a $R \in \mathcal{S}$ and $R' \in \{R_1, \ldots, R_k, R_1', \ldots, R_k'\}$ so that $R' \subset R$, then it is a partition of $[m]^d$ that consists of at most $2k$ rectangles that respect the hierarchical structure.
Moreover, it is easy to see that $f - g$ is flat on every rectangle in $\mathcal{S}$.
This completes the proof.
\end{proof}

\noindent Lemma \ref{lem:difference-of-histograms} immediately implies Corollary \ref{cor:difference-of-histograms}.

We also wish to instantiate these bounds for rectangles, and unions of at most $k$ rectangles.
Fortunately, the VC dimension of rectangles and unions is well-understood:
\begin{lemma}[c.f. Devroye \& Lugosi Lemma 4.1]
If $\rect$ is the set of axis-aligned rectangles in $\R^d$, then $\VC (\rect) = 2d$.
\end{lemma}

\begin{lemma}[c.f. Devroye \& Lugosi Exercise 4.1]
For any two sets of sets $\mathcal{A}_1, \mathcal{A}_2$, we have $\VC (\mathcal{A}_1 \cup \mathcal{A}_2) \leq \VC (\mathcal{A}_1) + \VC (\mathcal{A}_2) + 1$, where
\[
\mathcal{A}_1 \cup \mathcal{A}_2 = \{A_1 \cup A_2 : A_1 \in \mathcal{A}_1, A_2 \in \mathcal{A}_2 \} \; .
\]
\end{lemma}
Together, these two lemmas imply Corollary \ref{cor:VC-hierk}.

\begin{proof}[Proof of Corollary \ref{cor:VC-hierk}]
Let $\rect$ denote the set of all axis aligned rectangles in $[m]^d$. 
The above two lemmas immediately imply that the class $\rect_k = \bigcup_{i = 1}^k \rect$ has $\VC (\rect_k) \leq 2 d k + k$.
Since $\hier_k \subseteq \rect_k$, the result follows immediately.
\end{proof}

\subsection{Proof of Lemma \ref{lem:compute-A1}}
Our algorithm is given in Algorithm \ref{alg:compute-A1}.
For any rectangle $R \subseteq [m]^d$ (resp. $[m]^d$), we let $\vol(R)$ denote its measure  in $[m]^d$ (resp. $[0, 1]^d$).

\begin{algorithm*}[htb]
\begin{algorithmic}[1]
\Function{ComputeD1}{$\fhat, \hier, R, a$}
\State Let $\calT$ be the tree of rectangles in $\hier$ containing points in $\supp (\fhat) \cap R$.
\State For every rectangle $R'$ in $\calT$, let $c(R') = |\supp (\fhat) \cap R'|$ \label{line:1}
\State Let $b_1 = \max_{R' \in \calT} | c(R') - a \cdot \vol (R') |$, and let $R_1$ be the rectangle which achieves this maxima.
\State Let $b_2 = a \cdot \vol(R')$, where $R'$ is the rectangle with maximum volume not in $\calT$, and let $R_2$ be the rectangle which achieves this maxima. \label{line:2}
\State \textbf{return} $\max (b_1, b_2)$ and the corresponding $R_1$ or $R_2$
\EndFunction
\end{algorithmic}
\caption{Approximating with histograms by splitting.}
\label{alg:compute-A1}
\end{algorithm*}

\begin{proof}[Proof of Lemma \ref{lem:compute-A1}]
We first prove the claimed runtime bound.
Observe that $\calT$ has size at most $O(2^d s \log m)$, and by a simple recursive splitting procedure, can be generated in $O(2^d s \log m)$ time.
Similarly $c(v)$ can be computed for every node in $\calT$ in $O(s \log m)$ time overall.
Therefore $b_1$ can be computed in $O(2^d s \log m)$ time overall.
To compute $b_2$, it suffices to find the largest rectangle $R'$ in $\calT$ which does not have $2^d$ children in $\calT$, and to return $\vol (R') / 2^d$.
This again can be done by iterating over the tree once, so this takes time $O(2^d s \log m)$.
Therefore overall the algorithm runs in time $O(2^d s \log m)$.

We now show correctness of the algorithm.
Let $R'$ be the rectangle which achieves the maxima for the $\hier$-distance.
There are two cases.
if $R' \cap \supp (\fhat) \neq \emptyset$, then clearly it is considered in Line \ref{line:1} of \textsc{ComputeA1}, and its contribution is considered in the distance computation.
Otherwise, $R'$ must be a rectangle with maximum volume not in $\calT$, as otherwise we may increase the value of the maxima by taking such a rectangle.
Therefore it is considered in Line \ref{line:2}.
In either case, its contribution is considered, and thus the algorithm is correct.
\end{proof}

\subsection{Proof of Theorem \ref{thm:splitting}}
For any set $R$, let $\OPT_{k} (R) = \| h^* - f \|_{1, R}$ be the $\ell_1$-error incurred by $h^*$ to $f$ on $R$.
Similarly, let $\tOPT_{\hier, k} (R)$ be the $\hier_k$-error incurred by the best fit hierarchical $k$-histogram to $\fhat$ on $R$.
For any collection of sets $\calS$, let $\OPT_{k} (\calS) = \OPT_{k} (\cup_{S \in \calS} S)$ and let $\tOPT_{\hier, k} (\calS)$ be defined similarly.
We now have all definitions we need for the proof.

\begin{proof}[Proof of Theorem \ref{thm:splitting}]
The proof of the bounds on the number of pieces and runtime are nearly identical to the proofs of Lemmas \ref{lem:pieces} and \ref{lem:runtime}, so we omit them.
Thus it suffices to prove correctness.
This is also quite similar to the proof of correctness for $\ell_2$.
Let $h^*$ be an optimal partial hierarchical $k$-histogram fit to $\fhat$ in $\hier_k$ norm, and let $h$ be the output of our algorithm.
Let $\calT$ be the tree associated with $h$.
Let $\rect^*$ be the set of $k$ disjoint dyadic rectangles on which $h^*$ is supported, and let $\rect$ be the leaves of $\calT$.
Partition $\rect$ into three sets: 
\begin{align*}
\calF &= \{ R \in \rect: ~\mbox{$h^*$ is constant on $R$}~\} \\
\calJ_1 &= \{ R \in \rect: ~\mbox{$h^*$ is non-constant on $R$ and $e_R = 0$}~\} \\
\calJ_2 &= \{ R \in \rect: ~\mbox{$h^*$ is non-constant on $R$ and $e_R > 0$}~\} \; .
\end{align*}

\noindent
We will prove that the error is low on all three sets separately.

\paragraph{Error on $\cal{F}$}
First, we will prove that the error is low in $\calF$.
In fact, we will prove a more general lemma which will be useful later:
\begin{lemma}
\label{lem:flat-err}
Let $\fhat$ be an empirical distribution, and let $\kappa$ be arbitrary.
Let $\rect \in \hier_k$ be any collection of at most $\kappa$ disjoint rectangles in $\hier$, and let $g$ be the function which is, on every $R \in \rect$, equal to the constant function $\phi_{a, R}$ which minimizes $\| \fhat - \phi_{a, R} \|_{\hier, R}$.
Then, if $h^*$ is constant on every rectangle in $\rect$, we have
\[
\| \fhat - g \|_{\hier_{\kappa}, \rect} \leq 3 \| \fhat - h^* \|_{\hier_{\kappa}, \rect}
\]
\end{lemma}
\begin{proof}
By a triangle inequality, we have 
\[
\| \fhat - g \|_{\hier_{\kappa}, \rect} \leq \| \fhat - h^* \|_{\hier_{\kappa}, \rect} + \| h^* - g \|_{\hier_{\kappa}, \rect} \; .
\]
Observe that on every rectangle $R \in \rect$, both functions are constant.
Hence
\begin{align*}
\| h^* - g \|_{\hier_{\kappa}, \rect} &= \sum_{R \in \rect} |h^* (R) - g (R)| \\
&= \sum_{R \in \rect} \| h^* - g \|_{\hier, R}  \\
&\leq \sum_{R \in \rect} \| h^* - \fhat \|_{\hier, R} + \sum_{R \in \rect} \| g - \fhat \|_{\hier, R} \\
&\leq 2 \sum_{R \in \rect} \| h^* - \fhat \|_{\hier, R} \\
&\leq 2 \| \fhat - h^* \|_{\hier_k, \rect} \; .
\end{align*}
Putting these two inequalities together yields the desired estimate.
\end{proof}
As an immediate corollary of this lemma, we get that
\begin{equation}
\label{eq:err-F}
\| \fhat - h \|_{\hier_k, \calF} \leq 3 \cdot \tOPT_{\hier, k} (\calF) \; .
\end{equation}

\paragraph{Error on $\calJ_1$}
We next consider the error on $\calJ_1$.
We will require the following elementary fact, which follows immediately from the definition of $\| \cdot \|_{\hier_k}$:
\begin{fact}
\label{fact:A1-to-Ak}
For any $k \geq 1$, and for any $\rect$, we have $\| f \|_{\hier_k, \rect} \leq k \| f \|_{\hier, \rect}$.
\end{fact}
This immediately implies that 
\begin{equation}
\label{eq:err-J1}
\| \fhat - h \|_{\hier_k, \calJ_1} = 0 \; .
\end{equation}

\paragraph{Error on $\calJ_2$}
Thus it suffices to bound the error on $\calJ_2$.
By a triangle inequality, we have
\begin{align*}
\| \fhat - h \|_{\hier_k, \calJ_2} &\leq \tOPT_{\hier, k} (\calJ_2) + \| h^* - h \|_{\hier_k, \calJ_2} \\
&\leq  \tOPT_{\hier, k} (\calJ_2) + \| h^* - h \|_{1, \calJ_2} \\
&=  \tOPT_{\hier, k} (\calJ_2) + \sum_{R \in \calJ_2} \| h^* - h \|_{1, R} \; .
\end{align*}
For any $R \in \calJ_2$, let $\Gamma (R) + 1$ denote the number of values that $h^*$ takes on $R$.
Note that 
\begin{align*}
\| h^* - h \|_{1, R} &= \| h^* - h \|_{\hier_{\Gamma (R) + 1}, R} \\
&\leq \| h^* - \fhat \|_{\hier_{\Gamma (R) + 1}, R} + \| \fhat - h \|_{\hier_{\Gamma (R) + 1}, R} \\
&\leq \tOPT_{\hier, k} (R) + (\Gamma (R) + 1) \| \fhat - h \|_{\hier, R} \; .
\end{align*}

Observe that $R$ cannot be an indivisible rectangle, as then otherwise $e_R = 0$ and so $R \in \calJ_1$ or $R \in \calF$.
Therefore, in some iteration, there must be some $R'$ so that $R \subseteq R'$ so that $R'$ was not split in this iteration.
Let $\phi_{a, R'}$ be the optimal constant fit in $\hier$ distance to $\fhat$ on $R'$.
Let $A_1, \ldots, A_{(1 + \xi) k}$ be the rectangles which were split in this iteration.
Since these rectangles are disjoint, this means that $h^*$ can be non-constant on at most $k$ of them.
WLOG assume that $h^*$ is constant on $A_1, \ldots, A_{\xi k}$.
Let $g$ be the optimal fit in $\hier$ to $\fhat$ over each rectangle in $\mathcal{A} = \{ A_1, \ldots, A_{\xi k} \}$.
We then have
\begin{align*}
\| \fhat - h \|_{\hier, R} &\stackrel{(a)}{\leq} \| \fhat - \phi_{a, R'} \|_{\hier, R} \\
&\leq \| \fhat - \phi_{a, R'} \|_{\hier, R'} \\
&\leq \frac{1}{\xi k} \sum_{i = 1}^{\xi k} \| \fhat - g \|_{\hier, A_i} \\
&\stackrel{(b)}{\leq} \frac{1}{\xi k} \| \fhat - g \|_{\hier_{\xi k}, \mathcal{A}} \\
&\stackrel{(c)}{\leq} \frac{3}{\xi k} \| \fhat - h^* \|_{\hier_{\xi k}, \mathcal{A}} \\
&\stackrel{(d)}{\leq} \frac{3}{\xi^2 k} \tOPT_{\hier, k} \; .
\end{align*}
where (a) follows from the fact that $h$ is the optimal $\hier$ fit to $\fhat$ on $R$, (b) follows from the definition of $\| \cdot \|_{\hier_k}$, (c) follows from Lemma \ref{lem:flat-err}, and (d) follows from Fact \ref{fact:A1-to-Ak}.
Hence, overall we have 
\begin{align*}
\| h^* - h \|_{1, R} &\leq \tOPT_{\hier, k} (R) +  \frac{3}{\xi^2 k} \tOPT_{\hier, k} \; .
\end{align*}
Summing over the elements in $\calJ_2$, we obtain that 
\begin{align*}
\| h^* - h \|_{1, \calJ_2} &\leq \tOPT_{\hier, k} (\calJ_2) + \frac{3}{\xi^2 k} \tOPT_{\hier, k}  \sum_{R \in \calJ_2} (\Gamma (R) + 1) \\
&\leq \tOPT_{\hier, k} (\calJ_2) + \frac{6}{\xi^2} \tOPT_{\hier, k} \; .
\end{align*}
Hence overall, we have 
\begin{equation}
\label{eq:err-J2}
\| \fhat - h \|_{\hier_k, \calJ_2} \leq 2 \tOPT_{\hier, k} (\calJ_2) + \frac{6}{\xi^2} \tOPT_{\hier, k} \; .
\end{equation}
Combining (\ref{eq:err-F}), (\ref{eq:err-J1}), and (\ref{eq:err-J2}) and simplifying yields that
\[
\| \fhat - h \|_{\hier_k} \leq \left( 3 + \frac{6}{\xi^2} \right) \tOPT_{\hier, k}  \; ,
\]
as claimed.
\end{proof}

\subsection{Proof of Lemma \ref{lem:VC-histograms}}
\begin{proof}[Proof of Lemma \ref{lem:VC-histograms}]
Let $T$ be a finite set of size $r$. 
If our family can shatter $T$, then all $2^r$ subsets of $T$ must be expressible in the form

\begin{equation}
\label{eq:shattering}
T \cap \left( \bigcup_{i = 1}^{k'} A_i - \bigcup_{j = 1}^{k''} B_j \right) = \left( \bigcup_{j = 1}^{k'} A_j \cap T \right) - \left( \bigcup_{j = 1}^{k''} B_j \cap T \right)
\end{equation}

We now count the number of possible sets of the form $R \cap T$ for rectangles $R$. 
Observe that each face has a fixed normal, and for halfspaces $H$ with a fixed normal 
there are clearly at most $r$ possible sets $H \cap T$. 
$R$ has $2d$ faces and so the number of possible sets of the form $R \cap T$ 
is at most $r^{2d}$.
Hence, the number of sets of the form in (\ref{eq:shattering}) is at most $r^{4dk}$. This is smaller than $2^r$ when $r$ is a sufficiently large multiple of $kd \log (kd)$. 
Thus, the VC dimension is $O(kd \log (kd))$.
\end{proof}

\subsection{Proof of Lemma \ref{lem:arbitrary-to-hierarchical-rounding}}
\begin{proof}[Proof of Lemma \ref{lem:arbitrary-to-hierarchical-rounding}]
Let $R_1, \ldots, R_k$ be a partition of $[m]^d$ into $k$ disjoint rectangles so that $h^*$ is constant on each $R_i$.
For each $i$, let $R_i'$ be the smallest rectangle so that $R_i' \subseteq R$ and so that $R_i'$ contains every point in $\grid (\calX) \cap R_i$.
Let $h^*_p$ be the $k$-partial histogram so that for each $i$, $h^*_p (x) = h^* (x)$ for all $x \in R_i'$, and $h^*_p (x) = 0$ outside of $\bigcup_{i = 1}^k R_i'$.
We claim this function satisfies 
\[
\| f - h^*_p \|_1 \leq \OPT_k + c' \eps \; .
\]
Let $R = \bigcup_{i = 1}^{k} R_i'$.
Then, we have
\begin{align*}
\| f - h^*_p \|_1 &= \| f - h^*_p \|_{1, R} + \| f - h^*_p \|_{1, R^c}  \stackrel{(a)}{=} \| f - h^* \|_{1, R} + \| f \|_{1, R^c} \\
&\stackrel{(b)}{=} \OPT_k + 2 \| f \|_{\A_k, R^c} \stackrel{(c)}{=} \OPT_k + 2 \| f - \fhat \|_{\A_k, R^c} \stackrel{(d)}{\leq} \OPT_k + 2 c' \eps \; ,
\end{align*}
where (a) follows from the decomposibility of $\ell_1$, (b) follows from the definition of $h^*_p$, (c) follows since $R \in \A_{\kappa}$ and since $\fhat = 0$ on $R^c$, and (d) follows from (\ref{eq:deterministic-cond}).

The only remaining problem with this function is that it is not a distribution, namely, it does not integrate to $1$.
However, we know that 
$\left| \| h^*_p \|_1 - 1 \right| \leq \| h^*_p - f \|_1 \leq \OPT_k + 2 c' \eps$.
Hence, if we renormalize $h^*_p$ to make it integrate to $1$ (say, by adding mass uniformly to one rectangle), we lose at most an additional $\OPT_k + 2 c' \eps$ factor.
The claim follows then from an easy generalization of Lemma \ref{lem:arbitrary-to-hierarchical}, since the side length of $\grid (\calX)$ is $\poly (k, 1 / \eps)$.
\end{proof}

\subsection{Proof of Lemma \ref{lem:partial-minus-full}}
\begin{proof}
Let $Z = \{h(x) > g(x)\}$.
Let $R_1, \ldots, R_k$ be $k$ disjoint rectangles so that $h$ is constant on every $R_i$, and $h$ is supported on their union.
Let $R_1', \ldots, R_k'$ be the same for $g$, except that these sets form a partition of $[m]^d$.
Reminiscent of the proof of Lemma \ref{lem:correctness-l2}, partition $\rect = \{R_1, \ldots, R_k \}$ into two sets: $\calF$, the set of $R \in \rect$ so that $g$ is constant on $R$, and $\calJ$, the set of $R \in \rect$ so that $g$ has a jump on $R$.
Clearly $Z$ is the disjoint union of $Z_1 = Z \cap \cup_{R \in \calF} R$ and $Z_2 = Z \cap \cup_{R \in \calJ}$.
Moreover, $Z_1$ is immediately expressible as the disjoint union of at most $k$ dyadic rectangles: namely, the rectangles in $\calF$ on which $h(x) > g(x)$.
Thus, it suffices to show that $Z_2$ can be written as a disjoint union of at most $k$ disjoint rectangles.
But if $R \in \calJ$ then this means that $g$ partitions the rectangle.
Thus, $Z_2$ is exactly the set of $R_j'$ so that $h(x) > g(x)$ on $R_j'$, and $R_j' \subset R_i$ for some $i \in [k]$.
Hence $Z_2$ can also be written as the union of at most $k$ rectangles, and so $Z$ can be written as a union of at most $2k$ rectangles, which completes the proof.
%
\end{proof}

\subsection{Algorithm \ref{alg:ada-hist-l1}}
\begin{algorithm}[htb]
\begin{algorithmic}[1]
\Function{AdaptiveGreedySplit}{$\fhat, \xi$}
\State Let $\calX$ be the (multi-)set of points in the support of $\fhat$
\State \textbf{return} $\textsc{GreedySplit}(\fhat, \hier(\grid (\calX)), \xi)$
\EndFunction
\end{algorithmic}
\caption{Adaptive greedy splitting for histogram learning in $\ell_1$}
\label{alg:ada-hist-l1}
\end{algorithm}

\section{Learning Histograms in $\ell_2$-Distance}
In this section, we consider the problem of learning the best fit $k$-histogram in $\ell_2$-distance
to a unknown distribution over $[m]^d$ given sample access to the distribution.
The main result of this section is the following:
\begin{theorem}
\label{thm:l2-main}
Fix $\eps, \delta > 0$, $k \in \mathbb{Z}_+$, and let $\gamma > 0$ be a tuning parameter.
Let $f: [m]^d \to R$ be an arbitrary distribution.
There is an algorithm \textsc{GreedySplitL2} which takes $n = O(\log (1 / \delta) / \eps )$ samples from $f$
and outputs a hierarchical $(1 + \xi) 2^d k \log^{d + 1} m$-histogram $h$ so that with probability at least $1-\delta$
\[
\| h - g \|_2^2 \leq \left( 1 + \frac{1}{\xi} \right) \OPT_k + \eps \; ,
\]
where $\OPT_{k} = \min_{h} \| h - g \|_{2}^2$ and the minimum is taken over all $k$-histograms $h$.
Moreover, the algorithm runs in time $O (2^d n \log^2 m)$.
\end{theorem}

While the statement of this theorem does not quite obtain the guarantees in Theorem \ref{thm:l2-informal}, 
in that we have $\log m$ factors instead of $\log 1 / \eps$ factors, it is straightforward 
to use the same adaptive gridding techniques as we did for $\ell_1$ to replace these $\log m$ factors with $\log 1 / \eps $ factors.
Since the ideas are subsumed by those described for $\ell_1$, we omit these details for simplicity.

Our starting point is the following well-known statistical guarantee, 
which states that the empirical distribution is $\eps$-close to the true distribution 
in $\ell_2$-norm after roughly $O(1 / \eps^2)$ samples. 
\begin{fact}[folklore, see, e.g., \cite{ADHLS15}]
\label{fact:l2-estimate}
Fix $\eps, \delta > 0$.
Let $f: [m]^d \to \R$ be an arbitrary distribution, and let $\fhat = \fhat_n$ 
be the empirical distribution after $n = O (\log (1 / \delta) / \eps)$ independent samples from $f$.
Then, with probability $1 - \delta$, we have $\| \fhat - f \|_2^2 \leq \eps$.
\end{fact}

This fact states that the $\ell_2$ learning problem is purely algorithmic: it suffices to, given $\fhat$, find the best fit $k$-histogram approximation to $\fhat$ in $\ell_2$.
Then by a simple application of the triangle inequality, this will be an almost optimal fit to $f$ in $\ell_2$ as well.
The main challenge is to devise algorithms for this problem which exploit the sparsity of $\fhat$.

We will also make crucial use of the following fact, which follows from basic calculus.
Then, we have:
\begin{fact}[folklore]
\label{fact:mean}
Let $\fhat$ be an empirical distribution over $[m]^d$, and let $R \subseteq [m]^d$ be any set.
Then, the best constant fit to $\fhat$ in $\ell_2$ on $R$ is the flattening of $\fhat$ over $R$.
\end{fact}

\subsection{Greedy Splitting for Hierarchical Histograms in $\ell_2$-Distance}
Our main algorithmic result for the $\ell_2$-norm 
is a greedy splitting routine which finds a nearly optimal hierarchical histogram fit to a sparse function efficiently.
Throughout this section, we will let $\hier = \hier ([m]^d)$ be the full dyadic decomposition of the domain.
While it is not hard to adapt the techniques in this section to work with an adaptive grid, 
as we did for the $\ell_1$-distance, we will not do this here for simplicity of the presentation, 
as this is not the main focus of our paper.

Our main theorem is:
\begin{theorem}
\label{thm:splitting-l2}
Fix $k \in \mathbb{Z}_+$, and let $\gamma > 0$ be a tuning parameter.
Let $g: [m]^d \to R$ be an arbitrary function supported on at most $s$ points.
There is an algorithm \textsc{GreedySplitL2} which outputs a $(1 + \xi) 2^d k \log m$-hierarchical histogram $h$ so that
\[
\| h - g \|_2^2 \leq  \left( 1 + \frac{1}{\xi} \right) \OPT_{\hier, k} \; ,
\]
where
$
\OPT_{\hier, k} = \min_{h} \| h - g \|_{2}^2,
$
where the minimum is taken over all hierarchical $k$-histograms $h$.
Moreover, the algorithm runs in time $O (2^d s \log^2 m)$.
\end{theorem}
\noindent
Combining this with Lemma \ref{lem:arbitrary-to-hierarchical} and Fact \ref{fact:l2-estimate} immediately yields Theorem \ref{thm:l2-main}.
Thus, it suffices to prove this theorem.

Our algorithm, given formally in Algorithm \ref{alg:hist-l2}, is quite similar to Algorithm \ref{alg:hist-l1}.
We construct a tree of nested dyadic rectangles.
Initially, this tree contains only $[m]^d$.
Iteratively, we find the leaves of this tree with largest $\ell_2^2$ error to $g$, 
and we split these into all of its children, and we repeat this for $\log m$ iterations.
At the end, we return the flattening of $g$ over all the leaves in the final tree.

\begin{algorithm}[htb]
\begin{algorithmic}[1]
\Function{GreedySplitL2}{$g, \xi$}
\State Let $\calT$ be a subtree of the hierarchical tree, initially containing only the root.
\For{$\ell = 1, \ldots, \log m$}
	\For{each leaf $R \in \calT$}
		\State Let $a_R = g(R) / |R|$
		\State Let $e_R = \sum_{x \in R} (g(x) - a_R)^2$
	\EndFor
	\State Let $\setJ$ be the set of $(1 + \xi) k$ leaves $R \in \calT$ with largest $e_R$. \label{line:J-l2}
	\For{each $R \in \setJ$}
		\If{$R$ can be subdivided in $\hier$ and $e_R > 0$}
			\State Add all children of $R$ to $\calT$
		\EndIf
	\EndFor
\EndFor
\State \textbf{return} The flattening of $g$ for every leaf $R$ of $\calT$
\EndFunction
\end{algorithmic}
\caption{Algorithm for learning a hierarchical histogram in $\ell_2$}
\label{alg:hist-l2}
\end{algorithm}

We will prove this theorem in three parts.
First, we will prove a bound on the number of pieces of the output histogram (Lemma \ref{lem:pieces}).
Then, we will bound the runtime of the algorithm (Lemma \ref{lem:runtime}).
Finally, we will bound the error of the algorithm (Lemma \ref{lem:correctness-l2}).

We first bound the number of pieces in our output:
\begin{lemma}[Number of pieces]
\label{lem:pieces}
The output of \textsc{GreedySplitL2} has at most $(1 + \xi) 2^d k \log M$ pieces.
\end{lemma}
\begin{proof}
In each iteration, we split at most $(1 + \xi) k$ rectangles each into $2^d$ pieces, so we increase the number of pieces by at most $(1 + \xi) 2^d k$.
Since there are $\log M$ iterations, this immediately proves the bound.
\end{proof}

We now prove a bound on the runtime:
\begin{lemma}[Runtime]
\label{lem:runtime}
\textsc{GreedySplitL2} runs in time $O(2^d s \log^2 M)$.
\end{lemma}
\begin{proof}
In each iteration, we iterate over the number of rectangles currently in the tree, and we take $O(2^d s_R \log M)$ time per rectangle $R$, if $R$ contains $s_R$ points in the support of $g$.
Thus per iteration we do at most $O ((1 + \xi) 2^d s \log M)$ work, and there are $\log M$ iterations.
Multiplying these two terms yields the desired claim.
\end{proof}

Finally, we turn our attention to correctness:
\begin{lemma}
\label{lem:correctness-l2}
$\| h - g \|_2^2 \leq  \left( 1 + \frac{1}{\xi} \right) \OPT_k$.
\end{lemma}
\begin{proof}
Let $h^*$ be an optimal hierarchical $k$-histogram fit to $g$ in $\ell_2$ norm, and let $h$ be the output of our algorithm.
For any set $S \subseteq [m]^d$, let $\OPT_{\hier, k} (S) = \sum_{x \in S} (h^* (x) - g(x))^2$ be the $\ell_2$-squared error incurred by $h^*$ on $S$.
For any collection of sets $\calS$, let $\OPT_{\hier, k} (\calS ) = \OPT_{\hier, k} (\cup_{S \in \calS} S)$.

Let $\calT$ be the tree associated with $h$.
Let $\rect^*$ be the set of $k$ disjoint dyadic rectangles on which $h^*$ is supported, and let $\rect$ be the leaves of $\calT$.
Partition $\rect$ into three sets: 
\begin{align*}
\calF &= \{ R \in \rect: ~\mbox{$h^*$ is constant on $R$}~\} \\
\calJ_1 &= \{ R \in \rect: ~\mbox{$h^*$ is non-constant on $R$ and $e_R = 0$}~\} \\
\calJ_2 &= \{ R \in \rect: ~\mbox{$h^*$ is non-constant on $R$ and $e_R > 0$}~\} \; .
\end{align*}

\noindent
We will prove that the error is low on all three sets separately.

\paragraph{Error on $\cal{F}$}
First, we will prove that the error is low in $\calF$.
In fact, we will prove a more general lemma which will be useful later:
\begin{lemma}
\label{lem:flat-err-l2}
Let $g$ be arbitrary.
Let $\rect \in \hier_k$ be any union of at most $k$ disjoint rectangles in $\hier$, and let $\flatten{g}$ be the flattening of $g$ over the rectangles in $\rect$.
Then, if $h^*$ is constant on every rectangle in $\rect$, we have
\[
\| g - \flatten{g} \|_{2, \rect}^2 \leq \| g - h^* \|_{2, \rect}^2
\]
\end{lemma}
\begin{proof}
This follows immediately from Fact \ref{fact:mean}.
\end{proof}
\noindent
As an immediate corollary of this lemma, we get that
\begin{equation}
\label{eq:err-F-l2}
\| g - h \|_{2, \calF}^2 \leq \OPT_{\hier, k} (\calF) \; .
\end{equation}

\paragraph{Error on $\calJ_1$}
By definition, we have
\begin{equation}
\label{eq:err-J1-l2}
\| g - h \|_{2, \calJ_1}^2 = 0 \; .
\end{equation}

\paragraph{Error on $\calJ_2$}
Finally, we bound the error $\calJ_2$.
Fix any $R \in \calJ_2$.
Observe that $R$ cannot be an indivisible rectangle, as then otherwise $e_R = 0$ and so $R \in \calJ_1$ or $R \in \calF$.
Therefore, in some iteration, there must be some $R'$ so that $R \subseteq R'$ so that $R'$ was not split in this iteration.
Let $A_1, \ldots, A_{(1 + \xi) k}$ be the rectangles which were split in this iteration.
Because the rectangles are dyadic, they are disjoint.
Thus, $h^*$ can be non-constant on at most $k$ of them.
WLOG assume that $h^*$ is constant on $A_1, \ldots, A_{\xi k}$.
Let $q$ be the flattening of $g$ over $\mathcal{A} = \{ A_1, \ldots, A_{\xi k} \}$.
We then have
\begin{align*}
\| g - h \|_{2, R}^2 &\stackrel{(a)}{\leq} \| q - \flatten{g_{R'}} \|_{2, R'}^2 \\
&\leq \frac{1}{\xi k} \sum_{i = 1}^{\xi k} \| g - \flatten{g_{A_i}} \|_{2, A_i}^2 \\
&\stackrel{(b)}{\leq} \frac{1}{\xi k} \OPT_{\hier, k} \; .
\end{align*}
where (a) follows from the fact that $h$ is the optimal $\ell_2$ fit to $g$ on $R$, and (b) follows from Lemma \ref{lem:flat-err-l2}.

Summing over the elements in $\calJ_2$, we obtain that 
\begin{equation}
\label{eq:err-J2-l2}
\| g - h \|_{2, \calJ_2}^2 \leq \frac{1}{\xi} \OPT_{\hier, k} \; .
\end{equation}

Combining (\ref{eq:err-F-l2}), (\ref{eq:err-J1-l2}), and (\ref{eq:err-J2-l2}) and simplifying yields that
\[
\| g - h \|_{2}^2 \leq \left( 1 + \frac{1}{\xi} \right) \OPT_{\hier, k}  \; ,
\]
as claimed.
\end{proof}
\noindent 
Lemmas \ref{lem:pieces}, \ref{lem:runtime}, and \ref{lem:correctness-l2} together immediately imply Theorem \ref{thm:splitting-l2}.

\end{document}